\pdfoutput=1

\documentclass[11pt]{article}

\usepackage[final]{acl}

\usepackage{times}
\usepackage{latexsym}

\usepackage[T1]{fontenc}

\usepackage[utf8]{inputenc}

\usepackage{microtype}

\usepackage{inconsolata}

\usepackage{amstext}
\usepackage{amsmath}
\usepackage{amssymb}
\usepackage{subcaption}
\usepackage{graphicx}
\usepackage{booktabs}
\usepackage{multirow}
\usepackage{tikz-dependency}
\usepackage{amsthm}
\usepackage{bbm}
\usepackage{thmtools}
\usepackage{thm-restate}
\usepackage{stfloats}
\newtheorem{lemma}{Lemma}

\title{Layer-Condensed KV Cache for Efficient Inference of\\ Large Language Models}

\author{Haoyi Wu \and Kewei Tu\thanks{\; Corresponding author.} \\
School of Information Science and Technology, ShanghaiTech University \\
Shanghai Engineering Research Center of Intelligent Vision and Imaging \\
\texttt{\{wuhy1, tukw\}@shanghaitech.edu.cn}}

\begin{document}
\maketitle
\begin{abstract}
  Huge memory consumption has been a major bottleneck for deploying high-throughput large language models in real-world applications. In addition to the large number of parameters, the key-value (KV) cache for the attention mechanism in the transformer architecture consumes a significant amount of memory, especially when the number of layers is large for deep language models. In this paper, we propose a novel method that only computes and caches the KVs of a small number of layers, thus significantly saving memory consumption and improving inference throughput. 
  Our experiments on large language models show that our method achieves up to 26$\times$ higher throughput than standard transformers and competitive performance in language modeling and downstream tasks. 
  In addition, our method is orthogonal to existing transformer memory-saving techniques, so it is straightforward to integrate them with our model, achieving further improvement in inference efficiency.
  Our code is available at \url{https://github.com/whyNLP/LCKV}.
\end{abstract}

\section{Introduction}
High throughput and low latency are essential for deploying large language models (LLMs) in real-world applications \cite{tillet2019triton,kwon2023efficient}. However, the huge memory consumption of LLMs has been a major bottleneck, preventing a large batch size and high throughput generation. Among the memory-consuming components, the key-value (KV) cache is one of the most significant parts \cite{pope2023efficiently,zhang2023ho} that takes over 30\% of the GPU memory during deployment \cite{kwon2023efficient}. The KV cache is used to store the keys and values in each transformer layer during generation to avoid re-computation. Its memory consumption is proportional to both the sequence length and the number of layers.

There have been substantial works on reducing the memory consumption of the KV cache in LLMs. Most of them focus on compressing the KV cache by reducing the length of the cached KV sequence. For example, \citet{jiang-etal-2023-llmlingua,li-etal-2023-compressing,mu2023learning} compress the prompts to save the memory consumption. \citet{ren-etal-2023-context} incrementally compress a specified span of tokens into compact ones to reduce the KV cache length. \citet{xiao2024efficient,han2023lm} propose to store only the KVs of initial and recent tokens in the KV cache. \citet{zhang2023ho} propose a dynamic KV cache eviction policy to only keep a small portion of the KV cache in memory.

\begin{figure}[tb]
  \centering
  \begin{subfigure}{.22\textwidth}
    \centering
    \includegraphics[page=1,width=\textwidth,trim=360 150 350 130,clip]{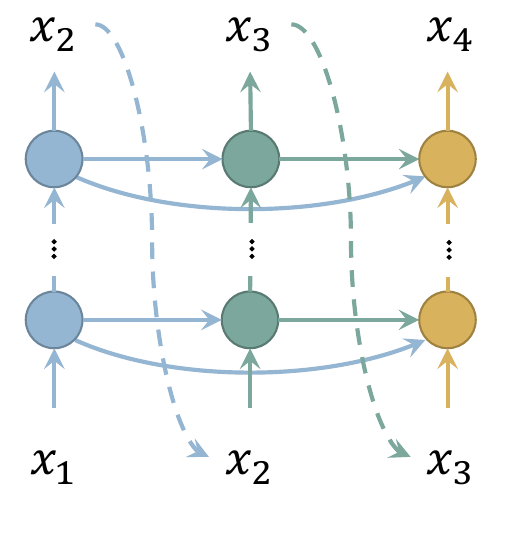}
    \caption{Standard transformer}
  \end{subfigure}
  \quad
  \begin{subfigure}{.22\textwidth}
    \centering
    \includegraphics[page=2,width=\textwidth,trim=360 150 350 130,clip]{figs/imgs.pdf}
    \caption{Our model}
  \end{subfigure}
  \caption{Illustration of a standard transformer decoder and our model. Each node represents one layer of transformer computation of one token. Each horizontal edge $a \rightarrow b$ denotes that the queries at $b$ are paired with the KVs at $a$.} 
  \label{fig:cover}
\end{figure}

In this paper, we propose to reduce the memory consumption of the KV cache from a novel perspective that is orthogonal to previous efforts: reducing the number of layers. Specifically, we present a new variant of transformer decoders in which queries of all layers are paired with keys and values of just the top layer, as illustrated in Figure~\ref{fig:cover}.
In this way, we only need to cache the keys and values of one layer (vs. tens of layers in a typical LLM), significantly saving memory consumption while introducing no computation overheads during inference. In fact, since we only need the keys and values of the top layer, we can omit the key-value computation and discard the key-value parameters for all the other layers, further improving the throughput and reducing the memory consumption as well as the mode size.

We draw our inspiration from the interpretation of the stacking layer structure of a transformer as an iterative process of improving token representation \cite{wu-tu-2023-probabilistic}. In this interpretation, the representation at the top layer is the most informative, so it makes sense to attend only to the top layer.
We also note the similarity of our idea to the cross-attention mechanism in a standard transformer encoder-decoder, in which all the decoder layers attend to the top encoder layer. However, applying this idea to a decoder has never been attempted before as far as we know.

Although our model presented above achieves very efficient inference, its performance in language modeling and downstream tasks degrades in comparison with standard transformers. 
Therefore, we further propose to retain standard attention for a small number of layers in our model, which slightly diminishes our saving of the KV cache memory consumption but leads to almost no performance degradation.


Another challenge faced by our model is training. When training a standard transformer decoder, the computation of all the tokens can be fully parallelized. In our model, however, the computation at each token depends on the top layer of the previous token, creating sequential dependencies that spoil parallel training.
We address the challenge by deriving an approximate training method that supports parallel training.

Our experiments on Llama \cite{touvron2023llama} show that our model achieves up to 32$\times$ larger batch sizes and up to 26$\times$ higher throughput than standard transformers for LLMs of 1B--30B parameters; at the same time, our model has competitive performance to standard transformers in language modeling and downstream tasks. 
We further empirically demonstrate that it is straightforward to integrate our model with other memory-saving techniques like StreamingLLM \cite{xiao2024efficient}, achieving further improvements in inference efficiency.

We summarize our contributions as follows: 1) we propose a new variant of transformer decoders that reduces the memory consumption of the KV cache by dramatically reducing the number of cached layers; 2) we make parallel training of our model feasible by designing a novel approximate training method; 3) we conduct extensive experiments to verify and analyze the effectiveness and efficiency of our method.

\section{Layer-Condensed KV Cache}

\subsection{Model}
\label{sec:model}
As shown in Figure~\ref{fig:cover}(b), we pair the queries of all layers with KVs of only the top layer, so that we do not have to cache or even compute KVs for layers other than the top layer, saving both memory consumption and computation.
Furthermore, since we no longer need to compute KVs for these layers, nor do we need to keep the weights $W_K, W_V$ that map hidden representations to KVs for these layers, thus also saving model parameters.

One problem with this method is that, since each token also attends to itself, we need its top-layer KVs for its attention computation at lower layers, but the top-layer cannot be computed until we finish the computation of all the lower layers. A straightforward solution to this cyclic dependency problem is to drop the attention of each token to itself, which is equivalent to masking the diagonal of the attention matrix. Now the first token of the sequence has nothing to attend to, so we just use zero vectors as dummy KVs in its attention computation. Note that even without self-attention of each token, its information can still be incorporated in its bottom-up computation thanks to residual connections.
Empirically, we find that the diagonal mask of the attention matrix does not affect the performance of the model.

It has been previously observed that transformers tend to attend to syntactic information in lower layers and semantic information in higher layers \cite{clark-etal-2019-bert}. Intuitively, applying KVs of the same layer to queries of all layers might break this pattern. Empirically, we do find that our method as described above underperforms standard transformers in language modeling and downstream tasks.
A simple yet effective solution to this problem is to retain standard attention for a small number of layers, which we call \emph{warmup layers}, and only apply our method to the rest of the layers. Inspired by the sandwich module of \citet{reid-etal-2021-subformer-exploring}, we propose to keep the top $w/2$ layers and the bottom $w/2$ layers as warmup layers. 
We empirically find that such a sandwich configuration outperforms alternative configurations and leads to almost no performance degradation compared with standard transformers. 

\subsection{Training}
\label{sec:training}

Though the inference process of our model is straightforward and almost identical to that of a standard transformer, i.e., decoding one token at a time from left to right, the training process of our model is more complicated. Since each token relies on the KVs of the top layer of previous tokens, it is impossible to train on a sequence of tokens in parallel as in a standard transformer.
Below we derive a parallel training process of our model in three steps.
For simplicity, we assume no warmup layers (i.e., $w=0$). It is straightforward to add warmup layers in the derived training process.

\subsubsection{From Sequential to Parallel Training}
During training, we minimize the cross entropy loss for each token. In our model, the computation of each token depends on the top-layer KVs of its previous tokens. Therefore, for a token sequence of length $n$, training is done sequentially on a computation graph consisting of $n$ passes of bottom-up transformer computation, as shown in Figure~\ref{fig:iter-train}(a).

In the following, we present a different training computation graph and show its equivalence to the original computation graph.
Specifically, we perform $n$ iterations of bottom-up transformer computation on all tokens in parallel, and in each iteration, we pair the queries of all layers with KVs of the top layer from the previous iteration, thus breaking the sequential dependencies within the iteration. We compute the cross entropy loss only after the last iteration. Note that the first iteration has no previous iteration, and we pair its queries with dummy KVs which are zero vectors. Figure~\ref{fig:iter-train}(b) shows the new computation graph.

\begin{restatable}{theorem}{train}
  \label{thm:train}
  The two computation graphs are equivalent in terms of model training.
\end{restatable}

\begin{figure}[tb]
  \centering
  \begin{subfigure}{.22\textwidth}
    \centering
    \includegraphics[page=3,width=\textwidth,trim=340 80 330 30,clip]{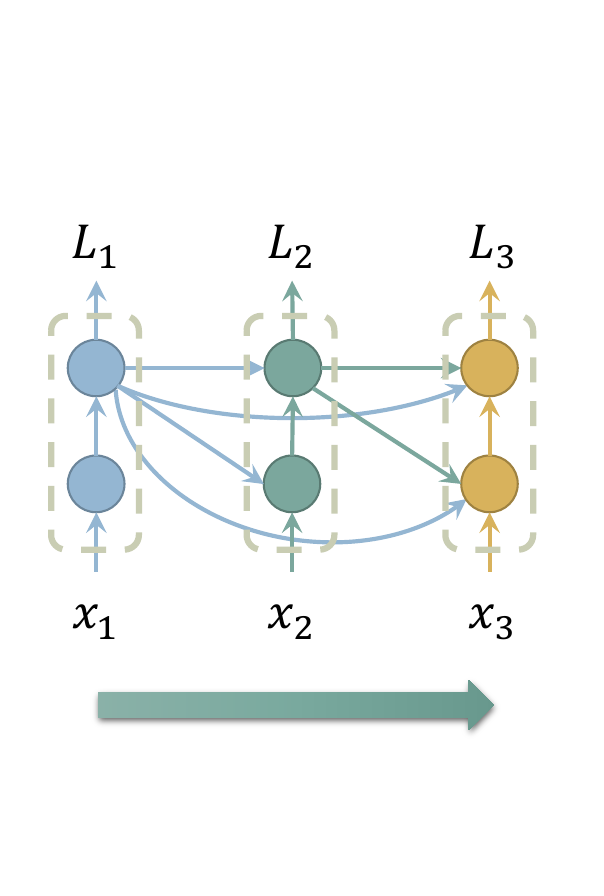}
    \caption{}
  \end{subfigure}
  \quad
  \begin{subfigure}{.22\textwidth}
    \centering
    \includegraphics[page=4,width=\textwidth,trim=360 80 310 30,clip]{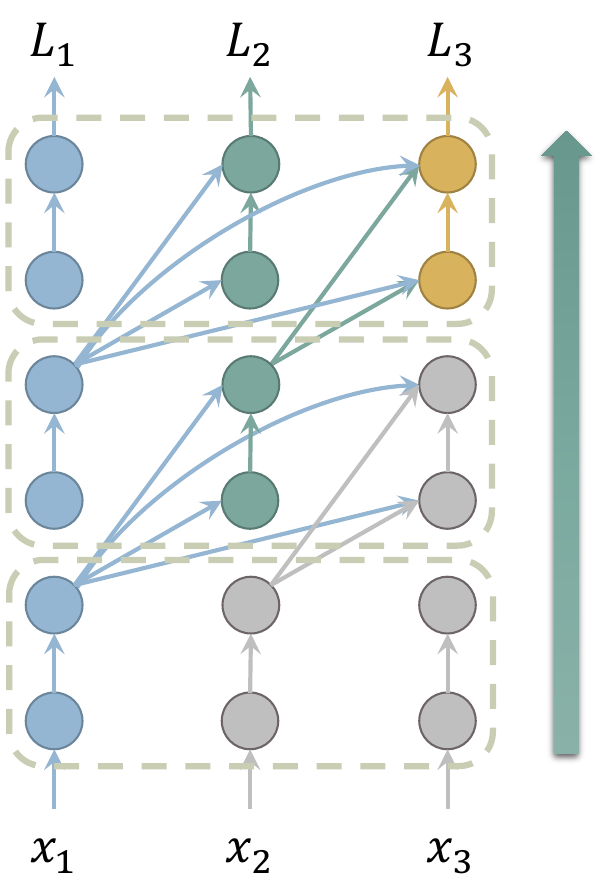}
    \caption{}
  \end{subfigure}
  \caption{Two equivalent computation graphs for training our model with two layers. (a) Sequential over $n$ tokens. (b) Parallel over $n$ tokens with $n$ iterations. Sub-graphs with the same color (except grey) represent identical data flow. Sub-graphs in grey are unused in loss computation.}
  \label{fig:iter-train}
\end{figure}

We leave the proof of the theorem to Appendix~\ref{apx:proof}. Here, we provide an intuitive explanation. We say the computation of a token (w.r.t. its hidden representations and top-layer KVs) is correct in an iteration of the second computation graph if and only if it is identical to that in the first computation graph. Since we have masked the diagonal of the attention matrix, the computation of the first token does not rely on any key or value (except for dummy zero vectors) and thus is correct in every iteration. For the second token, its computation relies on the KVs of the first token and thus is correct starting from the second iteration. In general, the computation of the $i$-th token relies on the KVs of the first $i-1$ tokens and by induction, it is correct starting from the $i$-th iteration. Therefore, after $n$ iterations, all the tokens are correctly computed. As a result, the computation sub-graphs of the cross-entropy losses of all the tokens are identical in the two graphs, and hence training processes following the two graphs are equivalent.

Essentially, the second computation graph replaces horizontal dependencies in the first graph with vertical dependencies, and it does not change the length of the longest dependency chain.
Consequently, although the second computation graph supports parallel training over all the tokens, it still requires the same $n$ iterations as the first graph. 
Next, we will trim the iterations first in terms of backpropagation and then in terms of forward propagation.

\subsubsection{Backpropagation: Gradient Stopping}
\label{sec:backpropagation}

We compute the cross entropy loss after the last iteration, which backpropagates through $n$ iterations, resulting in a large computation graph that is impossible to fit into GPU memory for large $n$.
To solve the issue, we follow the practice of gradient stopping in Transformer-XL \cite{dai-etal-2019-transformer} and backpropagate the loss only through the last $b$ iterations ($b \ll n$).

Notice that the KVs used in the last iteration come from the second-to-last iteration, so if we set $b=1$, then backpropagation would not reach the KVs and hence the model parameters used to calculate the KVs are not trained at all, which would result in a large performance degradation. We empirically find that with $b \geq 2$, the performance of our model is comparable with that of a standard transformer. To reduce memory consumption, we set $b=2$ by default, which means we only backpropagate through two iterations.

\subsubsection{Forward Propagation: Fast Convergence of KV}
\label{sec:fast-convergence}

With gradient stopping, the first $n-b$ iterations are solely used in forward propagation to compute KVs that are fed into the last $b$ iterations. When $n$ is large, it is still too costly to run $n-b$ iterations of forward propagation.
Fortunately, we observe that the values of KVs converge very fast over iterations and hence we do not have to run $n-b$ iterations to obtain the final KVs.
Figure~\ref{fig:kv-mse-1} shows the convergence of KVs of a randomly initialized model with the same configuration as TinyLlama \cite{zhang2024tinyllama}. 
The input is a randomly picked segment of text from the MiniPile \cite{kaddour2023minipile} dataset with 2048 tokens (i.e., $n=2048$). We measure the change of KVs over consecutive iterations using the mean squared error. 
As can be seen, KVs converge after only a few tens of iterations, with more warmup layers leading to faster convergence.
Therefore, we use $m$ iterations ($m \ll n$) to approximate the KVs of $n-b$ iterations. We empirically find that $m=7$ is sufficient for model training and larger values of $m$ do not further improve the performance.

\begin{figure}[tb]
  \centering
  \includegraphics[page=1,width=0.44\textwidth,trim=0 20 10 50,clip]{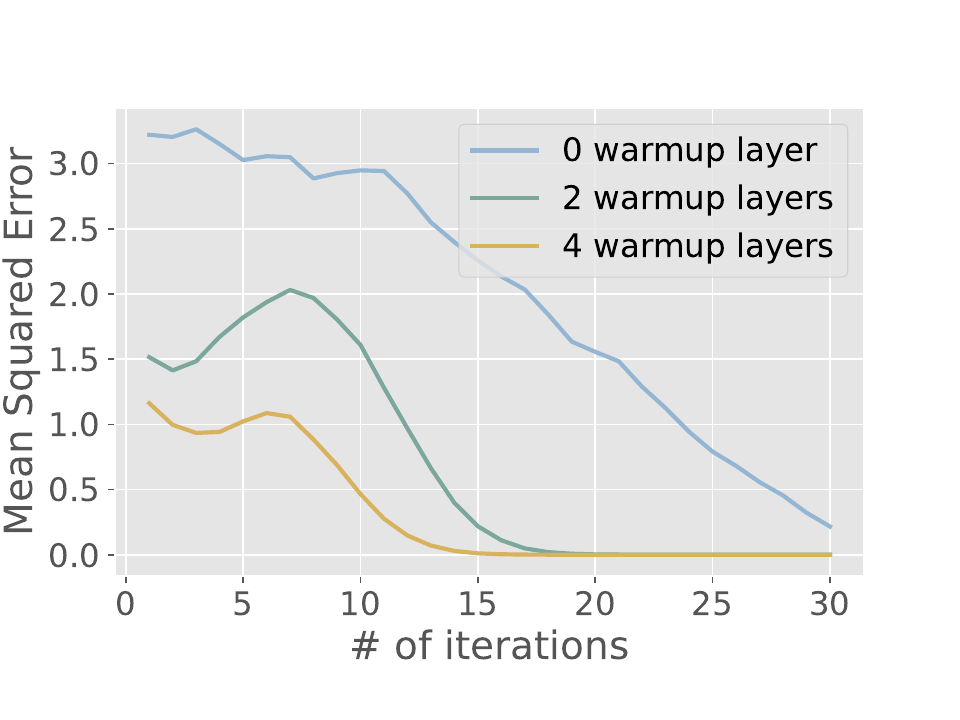}
  \caption{MSE of the KV before and after the $i$th iteration. The model is randomly initialized and tested with 2048 tokens.}
  \label{fig:kv-mse-1}
\end{figure}

\subsection{Inference with Prompts}
\label{sec:prompt}

It is straightforward to employ our model for autoregressive generation. However, our model cannot do parallel encoding of prompts like standard transformers for the same reason as it cannot do parallel training. Fortunately, since iterative computation of KVs is fast to converge, we can just perform iterative computation for the prompts for $m+b$ iterations. Typically, $m+b$ is far less than the number of tokens to generate, and thus the extra time spent in encoding the prompts is negligible.

\section{Experiments}
\label{sec:experiments}

We empirically verify the effectiveness of our method on the Llama model \cite{touvron2023llama}. We show that our method achieves significant memory reduction and throughput improvement as well as competitive performance in language modeling and downstream tasks compared with standard transformers. We also show that our method could effectively integrate with other memory-saving techniques.

\subsection{Generation Throughput}
\label{sec:throughput}

\begin{table*}[tb]
\centering
\small \setlength{\tabcolsep}{4pt}
\begin{tabular}{@{}ccccccccc@{}}
\toprule
\multirow{2}{2em}{GPU}      & \multirow{2}{*}{Model Size}                                 & \multirow{2}{*}{Seq. Length} & \multicolumn{3}{c}{Batch Size}                                                                                       & \multicolumn{3}{c}{Throughput (tokens/s)}                                                                              \\ \cmidrule(l){4-9} 
                          &                                                             &                              & Llama & \begin{tabular}[c]{@{}c@{}}Ours\\ $w=2$\end{tabular} & \begin{tabular}[c]{@{}c@{}}Ours\\ $w=10$\end{tabular} & Llama   & \begin{tabular}[c]{@{}c@{}}Ours\\ $w=2$\end{tabular} & \begin{tabular}[c]{@{}c@{}}Ours\\ $w=10$\end{tabular} \\ \midrule
\multirow{8}{2em}{RTX 3090} & \multirow{2}{*}{1.1B}                                       & 5+8187                       & 48    & 384 (8$\times$)                                      & 119 (2.5$\times$)                                    & 1424.96 & 4113.37 (2.9$\times$)                               & 2374.05 (1.7$\times$)                                \\
                          &                                                             & 5+2043                       & 239   & 1150 (4.8$\times$)                                  & 289 (1.2$\times$)                                    & 5142.86 & 10033.40 (2.0$\times$)                              & 7239.92 (1.4$\times$)                                \\ \cmidrule(l){2-9} 
                          & \multirow{5}{*}{7B}                                         & 5+8187                       & 1     & 12 (12$\times$)                                      & 4 (4$\times$)                                         & 32.02   & 151.91 (4.7$\times$)                                & 83.80 (2.6$\times$)                                  \\
                          &                                                             & 2048+2048                    & 2     & 23 (11.5$\times$)                                    & 8 (4$\times$)                                         & 56.98   & 171.65 (3.0$\times$)                                & 119.68 (2.1$\times$)                                 \\
                          &                                                             & 5+2043                       & 5     & 64 (12.8$\times$)                                    & 16 (3.2$\times$)                                      & 140.88  & 534.02 (3.8$\times$)                                & 315.38 (2.2$\times$)                                 \\
                          &                                                             & 512+512                      & 9     & 95 (10.6$\times$)                                   & 32 (3.6$\times$)                                     & 225.31  & 378.89 (1.7$\times$)                                & 380.60 (1.7$\times$)                                 \\
                          &                                                             & 512+1024                     & 7     & 72 (10.3$\times$)                                   & 16 (2.3$\times$)                                     & 174.11  & 401.92 (2.3$\times$)                                & 310.05 (1.8$\times$)                                 \\ \cmidrule(l){2-9} 
                          & \begin{tabular}[c]{@{}c@{}}30B\\ (CPU-offload)\end{tabular} & 512+1024                     & 4     & 83 (20.8$\times$)                                   & 23 (5.8$\times$)                                     & 0.23    & 5.99 (26.0$\times$)                                 & 1.63 (7.1$\times$)                                   \\ \midrule
\multirow{2}{2em}{A100}     & 7B                                                          & 2048+2048                    & 15    & 128 (8.5$\times$)                                   & 42 (2.8$\times$)                                      & 141.10  & 421.02 (3.0$\times$)                                & 315.09 (2.2$\times$)                                 \\ \cmidrule(l){2-9} 
                          & 30B                                                         & 2048+2048                    & 1     & 32 (32$\times$)                                      & 8 (8$\times$)                                         & 14.10   & 108.29 (7.7$\times$)                                & 77.65 (5.5$\times$)                                  \\ \bottomrule
\end{tabular}
\caption{Maximum generation batch size and throughput on an RTX 3090 (24GB) and an A100 (80GB) GPU respectively with different sequence lengths. Following \citet{zhang2023ho}, we use ``$x + y$'' to denote a prompt length of $x$ and a generation length of $y$.} 
\label{tab:throughput}
\end{table*}

We test our method with 1.1B, 7B, and 30B parameters on an NVIDIA GeForce RTX 3090 (24GB) GPU and an NVIDIA A100 (80GB) GPU respectively. The 1.1B model configuration follows that of TinyLlama \cite{zhang2024tinyllama} and the 7B and 30B model configuration follows that of the original Llama \cite{touvron2023llama}. We set $m=7, b=2$ and $w=\{2, 10\}$. Our implementation is based on HuggingFace Transformers \cite{wolf-etal-2020-transformers} with kernel replacement with FlashAttention 2 \cite{dao2023flashattention}, fused RMS norm, fused cross-entropy and fused SwiGLU.

Following FlexGen \cite{pmlr-v202-sheng23a}, the evaluation is conducted in an end-to-end fashion. For a prompt of length $s$, we let the model generate output sequences of length $n$ with batch size $b$. The latency $t$ is defined as the total time in seconds spent in processing the prompts and generating all the $bn$ tokens. The generation throughput is defined as $bn/t$ tokens per second.

Table~\ref{tab:throughput} compares the maximum batch sizes and throughput of standard Llama models and our models on the two types of GPUs. Note that some of the sequence lengths exceed the maximum input lengths that the models are trained on, but that does not affect batch size and throughput measurement. We still benchmark the batch sizes and throughput to show the potential of our method on other models allowing larger sequence lengths. It can be seen that our method achieves significantly larger batch sizes and higher throughput on all of the settings.

\begin{figure}[tb]
  \centering
  \includegraphics[page=1,width=0.44\textwidth,trim=0 20 0 50,clip]{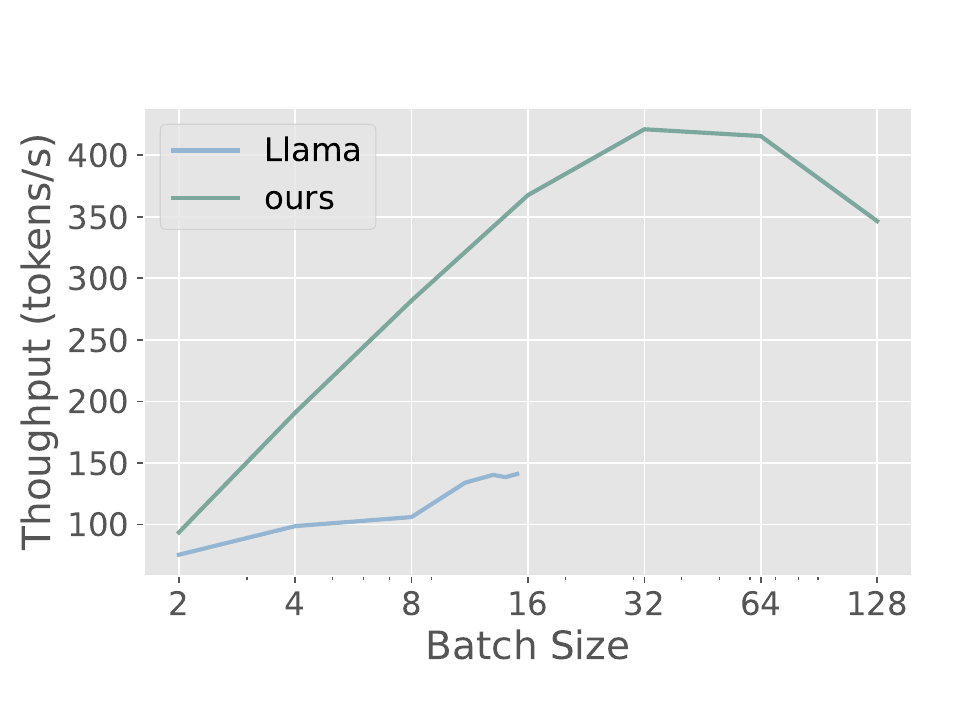}
  \caption{Throughput of 7B Llama and our model w.r.t. the batch size.}
  \label{fig:throughput}
\end{figure}

Notice that the maximum throughput is not necessarily achieved using the maximum batch size. Figure~\ref{fig:throughput} shows the throughput of 7B Llama and our model on an A100 GPU w.r.t. the batch size. The prompt length and generation length are both 2048. We find that when the batch size grows larger than 32, the throughput no longer increases and even decreases with a batch size of 128. This may indicate that the model operations are turning from memory-bound into compute-bound \cite{dao2022flashattention} and the throughput is limited by the computation power. Further improvement in throughput may require more efficient computation kernels instead of larger batch sizes.

We also would like to mention that the increased throughput is not necessarily due to the increased batch size. As shown in Figure~\ref{fig:throughput}, our model has much higher throughput than Llama even with the same batch sizes. We leave the detailed analysis to Appendix~\ref{apx:throughput}.

\subsection{Model Performance}
\label{sec:performance}

\begin{table*}[tb]
\centering
\begin{tabular}{@{}lcccccccc@{}}
\toprule
Model                   & HellaSwag & Obqa & WinoGrande & ARC-c & ARC-e & BoolQ & PIQA  & Avg   \\ \midrule
TinyLlama               & 44.58     & 30.2 & 50.99      & 25.00 & 46.38 & 60.46 & 68.93 & 46.65 \\
Ours ($w=2$)            & 42.22     & 30.6 & 49.64      & 24.74 & 43.10 & 61.38 & 66.49 & 45.45 \\
Ours ($w=10$)           & 44.74     & 31.0 & 51.70      & 24.83 & 46.38 & 61.38 & 67.90 & 46.84 \\ \bottomrule
\end{tabular}
\caption{Zero-shot accuracy on commonsense reasoning tasks.}
\label{tab:downstream}
\end{table*}

\begin{table}[tb]
\centering
\begin{tabular}{@{}lc@{}}
\toprule
Model                   & Dev ppl. \\ \midrule
TinyLlama               & 9.219    \\
Ours ($w=2$)            & 9.746    \\
Ours ($w=10$)           & 9.265    \\ \bottomrule
\end{tabular}
\caption{Perplexity on a 10M subset of the development set of SlimPajama.}
\label{tab:perplexity}
\end{table}

To evaluate the performance of our model in language modeling and downstream tasks, we pre-train from scratch two 1.1B models with $m=7$, $b=2$ and $w=\{2,10\}$. We use TinyLlama as our baseline, whose size is also 1.1B. We pre-train the models on a 100B subset of the SlimPajama dataset \cite{cerebras2023slimpajama}. The training details are consistent with those of TinyLlama \cite{zhang2024tinyllama}. All models are trained with AdamW \cite{loshchilov2018decoupled} with $\beta_1 = 0.9$ and $\beta_2 = 0.95$. The batch size is 2M tokens. We use a cosine learning rate schedule with a maximum learning rate of $4.0 \times 10^{-4}$ and a warmup of 200 steps. The final learning rate is $4.0 \times 10^{-5}$. We use a weight decay of 0.1 and gradient clipping of 1.0. The models are trained on 128 NVIDIA A800 (80GB) GPUs.

During evaluation, we perform inference in the standard left-to-right fashion instead of using the method of Section~\ref{sec:prompt}. We report the perplexity on a 10M subset of the development set of SlimPajama. We also test the zero-shot performance on commonsense reasoning tasks following \citet{zhang2024tinyllama}, including Hellaswag \cite{zellers-etal-2019-hellaswag}, OpenBookQA \cite{mihaylov-etal-2018-suit}, WinoGrande \cite{sakaguchi2021winogrande}, ARC-Easy and ARC-Challenge \cite{clark2018think}, BoolQ \cite{clark-etal-2019-boolq}, and PIQA \cite{bisk2020piqa}. The tests are conducted using the lm-eval-harness framework \cite{eval-harness}.
For these tasks, we encode the prompts with the same number of iterations ($m+b=9$) as in training. 

Table~\ref{tab:downstream} and \ref{tab:perplexity} show the results. The performance of our models is comparable to that of TinyLlama. In particular, our model with $w=10$ has almost no performance degradation, while achieving significantly higher generation throughput as evaluated in Section~\ref{sec:throughput}. Our model with $w=2$ has a small but noticeable decrease in performance for most of the tasks, but it achieves even higher increase in throughput. 

Despite the competitive performance and higher inference efficiency of our models, we note that pre-training our model costs about 3 times the time of pre-training TinyLlama with the same amount of data due to the iterative training process. Nevertheless, we believe that in most scenarios, a speedup in inference is worth a slowdown in training which is a one-time process.

\subsection{Integration with StreamingLLM}

\begin{figure*}[tb]
  \centering
  \includegraphics[page=1,width=\textwidth,trim=140 0 140 0,clip]{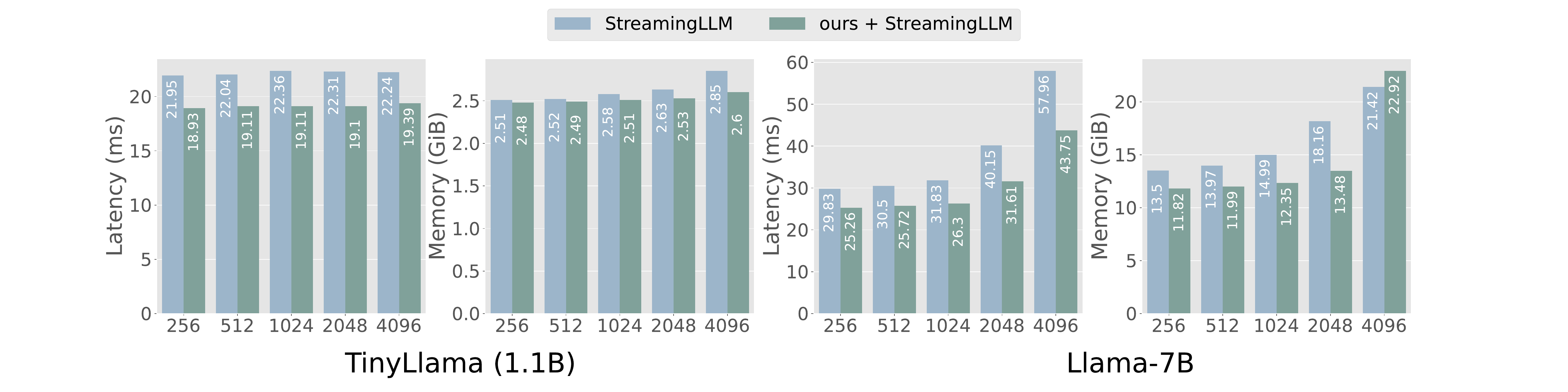}
  \caption{Comparison of latency per token and memory consumption of StreamingLLM and our model ($w=10$) integrated with StreamingLLM w.r.t. different cache sizes.}
  \label{fig:streaming}
\end{figure*}

We have previously mentioned that our method is orthogonal to other memory-saving techniques and can be easily integrated with them. Here we integrate our method with StreamingLLM \cite{xiao2024efficient}. StreamingLLM employs an attention sink that only preserves the KV cache of the first few tokens (four by default) and recent tokens, which empowers LLMs to process infinite-length inputs.


As shown in Figure~\ref{fig:streaming}, the integration of StreamingLLM and our model ($w=10$) achieves lower latency and memory consumption compared to the original StreamingLLM on different cache sizes (numbers of cached tokens). 


\begin{figure}[tb]
  \centering
  \includegraphics[page=1,width=0.44\textwidth,trim=0 10 0 40,clip]{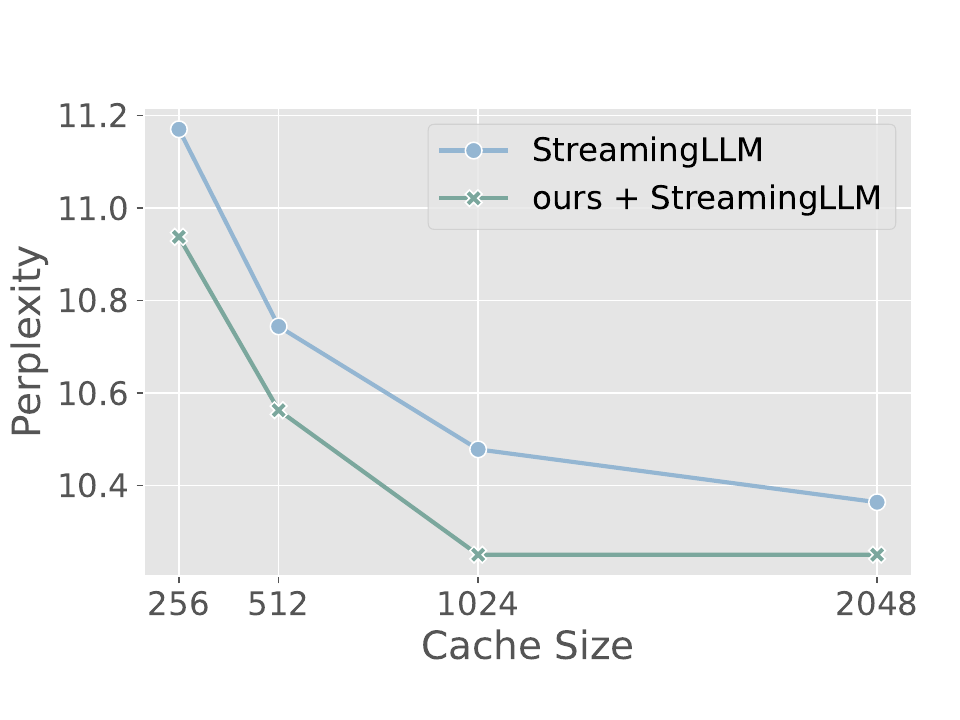}
  \caption{Comparison of StreamingLLM and our model integrated with StreamingLLM w.r.t. the cache size. We use 4 initial tokens for all settings. The results are collected on the first text sample of PG19 \cite{Rae2020Compressive}.}
  \label{fig:streaming-loss}
\end{figure}

\begin{figure}[tb]
  \centering
  \includegraphics[page=1,width=0.44\textwidth,trim=0 0 0 30,clip]{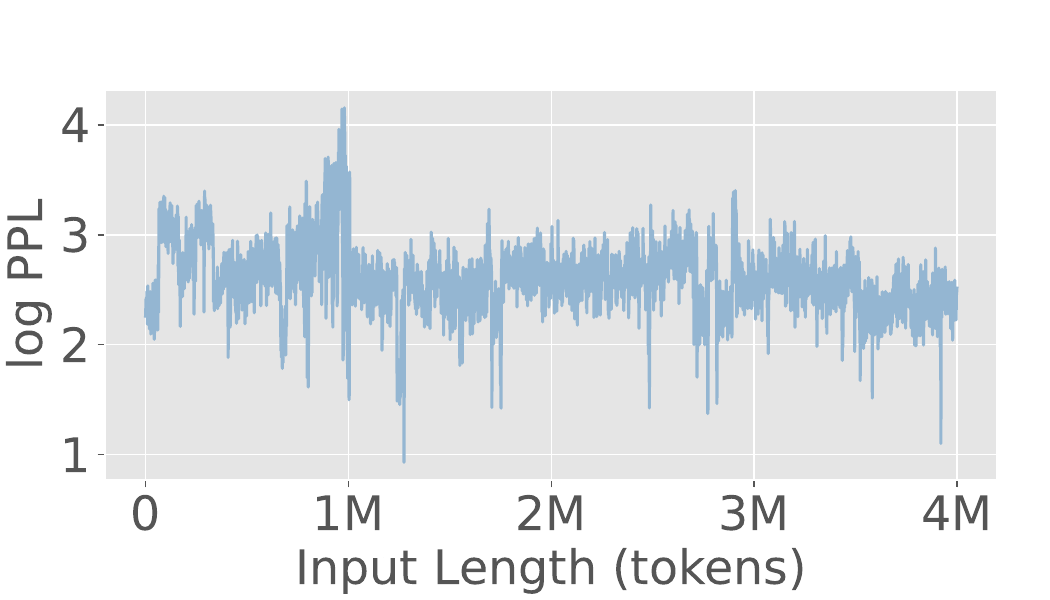}
  \caption{Language modeling perplexity of our model integrated with StreamingLLM on texts with 4M tokens. Following \citet{xiao2024efficient}, we use the concatenated test set of the PG19 dataset \cite{Rae2020Compressive} as the input.}
  \label{fig:streaming-4M}
\end{figure}

We further showcase that integration with our method does not hinder the ability of StreamingLLM to process infinite-length tokens. Specifically, we integrate StreamingLLM into our model ($w=10$) trained in Section~\ref{sec:performance} with different cache sizes and find that the integrated model achieves even lower perplexities than StreamingLLM as shown in Figure~\ref{fig:streaming-loss}. We also let the model handle inputs with a sequence length of four million tokens. As shown in Figure~\ref{fig:streaming-4M}, the integrated model can effectively process the input with the perplexity remaining stable.

\section{Analyses}
\label{sec:analysis}

In this section, we empirically analyze design choices in our method. 
For experiment details, please refer to Appendix~\ref{apx:details}.

\subsection{The Sandwich Configuration}
\label{sec:sandwich}

In Section~\ref{sec:model}, we propose to add some warmup layers to improve model performance. Note that inference efficiency is determined by the number of warmup layers $w$ and not by their placement. Here we ask the following question: \textit{with the number of warmup layers fixed, where should we place the warmup layers to achieve the best performance?}

\begin{table}[tb]
\centering
\begin{tabular}{@{}cccc@{}}
\toprule
\multirow{2}{*}{Model Size} & \multicolumn{3}{c}{Dev ppl.}    \\ \cmidrule(l){2-4} 
                            & all-bottom & all-top & sandwich \\ \midrule
50M                         & 14.556    & 221.850   & 14.069    \\
1.1B                        & 7.668    & 9.098       & 7.381   \\ \bottomrule
\end{tabular}
\caption{Model performance with different warmup layer placements ($w=2$).}
\label{tab:sandwich}
\end{table}

Table~\ref{tab:sandwich} shows the model performance in language modeling with two warmup layers (i.e., $w=2$) that are placed at the bottom, at the top, and with the default sandwich configuration. It can be seen that the sandwich style performs the best. A possible explanation is that top and bottom layers serve different functionalities (e.g., semantic vs. syntactic) and it makes more sense to warm up both than just one of them. 

\subsection{Number of Warmup Layers}
\label{sec:warmup}

Warmup layers serve as a bridge between the standard transformer and our model. The more warmup layers we keep, the more similar it is to the standard transformer, and the less memory we save. In this section, we ask the following question: \textit{how does the number of warmup layers $w$ affect the model performance and throughput?}

\begin{figure}[tb]
  \centering
  \includegraphics[page=1,width=0.46\textwidth,trim=0 20 0 30,clip]{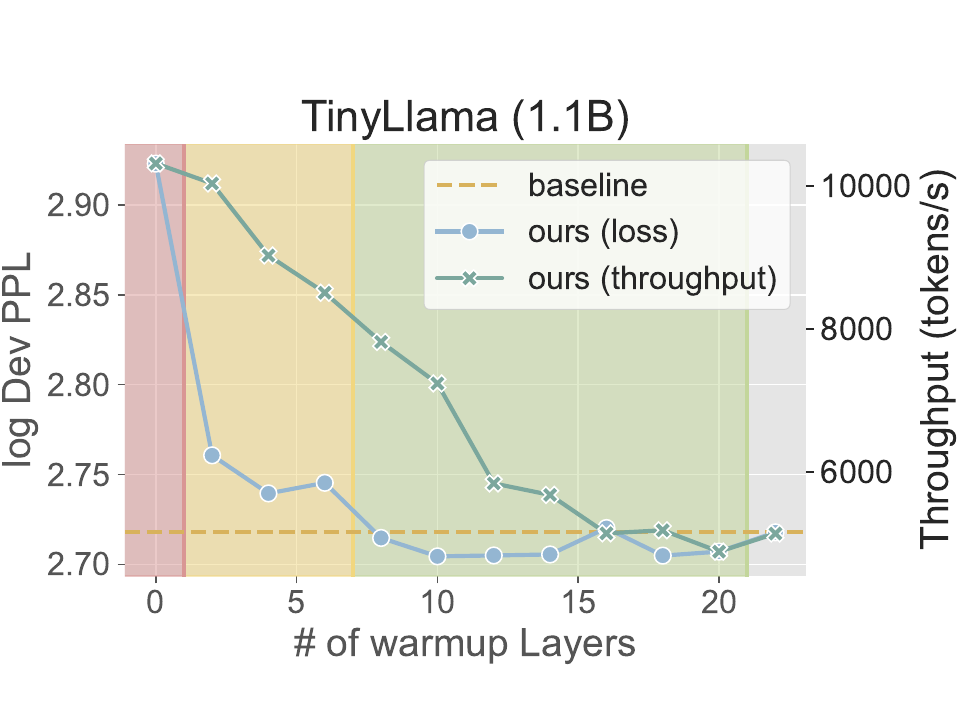}
  \caption{Effect of the number of warmup layers on model performance and throughput. The right-most point ($w=22$) denotes that all the layers are warmup layers and hence our model becomes exactly the standard transformer (the baseline). The throughput is tested on an RTX 3090 GPU with prompt length 5 and generation length 2043.}
  \label{fig:warmup}
\end{figure}

We test the model with 1.1B (22 layers) parameters and different numbers of warmup layers (Figure~\ref{fig:warmup}). Surprisingly, we find that the log dev perplexity does not monotonically decrease with the number of warmup layers. Without warmup layers (the red region), the model performance is significantly worse than that of the standard transformer. With only a few warmup layers (the yellow region), the model performance is greatly improved and close to that of the standard transformer, but its throughput is decreased at the same time. With more warmup layers (the green region), the model even outperforms the standard transformer, while suffering from further decrease in throughput.

The results point to a trade-off between model performance and throughput that is controllable by the number of warmup layers. If pursuing a high throughput, one could keep only a few warmup layers (the yellow region). If good performance is crucial, one could keep more warmup layers (the left part of the green region).

\subsection{Convergence of KV}
\label{sec:convergence}

In Section~\ref{sec:fast-convergence}, we have shown that KVs converge very fast for a random model and hence we use $m \ll n$ iterations to compute KVs. Here, we ask the following questions: \textit{how fast do KVs converge for a trained model and what value shall we pick for $m$?}


\begin{figure}[tb]
  \centering
  \includegraphics[page=1,width=0.44\textwidth,trim=0 60 0 60,clip]{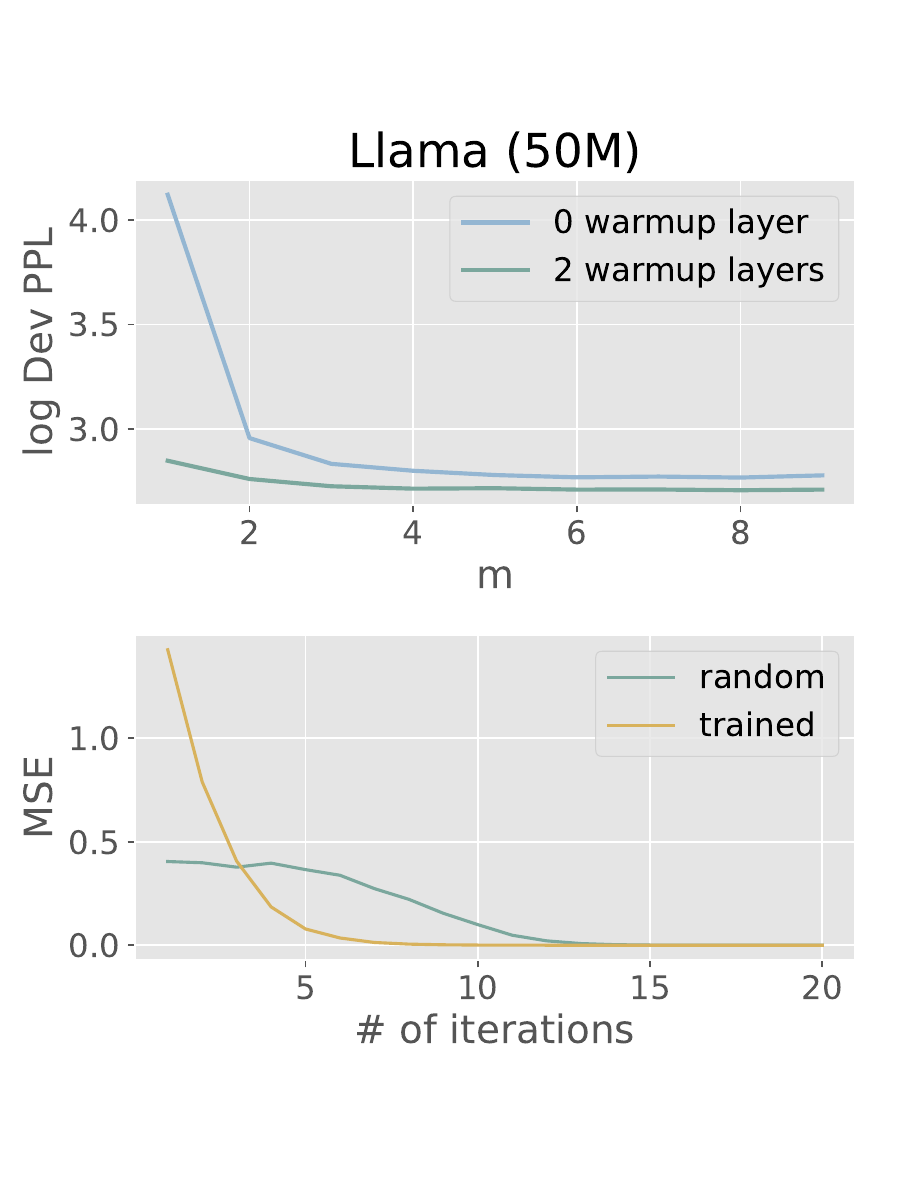}
  \caption{(Top) Experiment on a Llama (50M) with different values of $m$ during training. (Bottom) MSE of the KV before and after the $i$th iteration. The models are tested with 1024 tokens.}
  \label{fig:train-iter}
\end{figure}
%

We measure the convergence of KVs of a 50M Llama (Figure~\ref{fig:train-iter}, bottom) and the 1.1B model ($w=2$) pre-trained in Section \ref{sec:performance} (Figure~\ref{fig:kv-mse-2}). It can be seen that while a random model requires 15--20 iterations to converge, a trained model requires far fewer iterations. This hints at a small value of $m$ especially during late stages of training.
We then evaluate model performance when trained with different values of $m$ (Figure~\ref{fig:train-iter}, top). It can be seen that the model performance converges with $m \geq 6$, with more warmup layers leading to faster convergence. This justifies our default choice of $m=7$.

\begin{figure}[tb]
  \centering
  \includegraphics[page=1,width=0.44\textwidth,trim=0 20 30 30,clip]{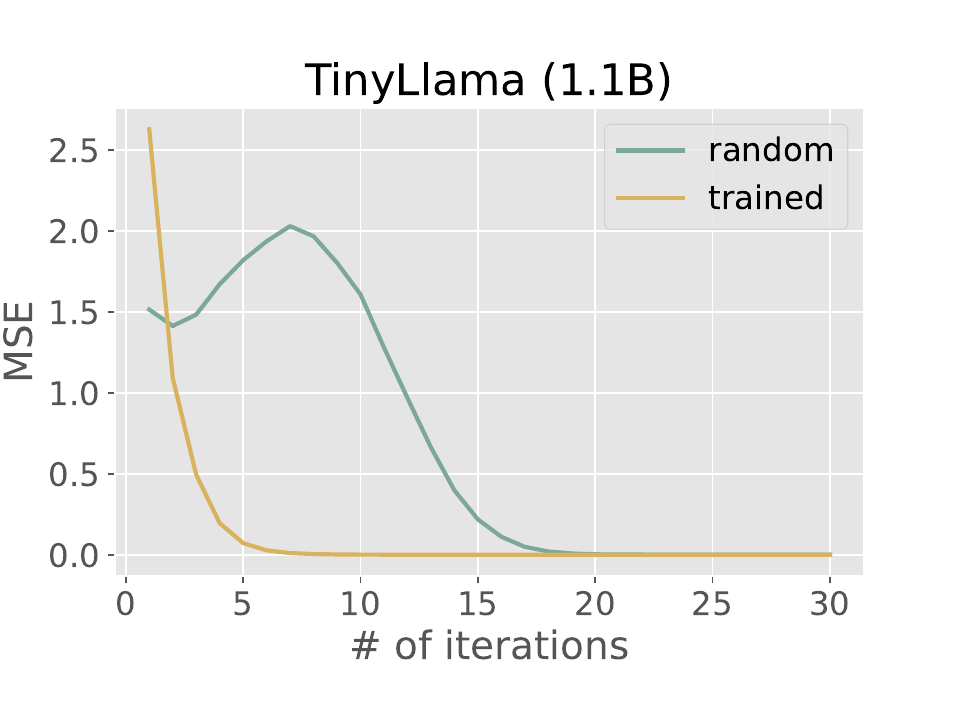}
  \caption{MSE of the KV before and after the $i$th iteration. The models are tested with 2048 tokens.}
  \label{fig:kv-mse-2}
\end{figure}

\section{Related Work}
Extensive research has been done on reducing KV cache for efficient inference of LLMs. With the exception of vLLM \cite{kwon2023efficient}, which proposes paged attention to reduce memory fragmentation of the KV cache from a system perspective, most recent works focus on compressing the KV cache by reducing the length of the cached KV sequence. \citet{jiang-etal-2023-llmlingua,jiang2023longllmlingua} accelerate model inference by compressing document-level prompts into short prompts. \citet{li-etal-2023-compressing} remove the redundancy in the input context. \citet{mu2023learning} train gist tokens to replace the reusable system prompts. \citet{ren-etal-2023-context} incrementally compress a specified span of tokens into compact ones to reduce the KV cache length. \citet{liu2023scissorhands} find that only pivotal tokens have a significant influence at a future step, so pruning unimportant tokens does not affect the performance. \citet{ge2023model} argue that the attention heads could be classified into different types and propose to apply different KV pruning strategies to different types of attention heads. \citet{xiao2024efficient,han2023lm} find that only initial tokens and recent tokens are crucial and propose to store only the KVs of these tokens to enable infinite-length context for LLMs. \citet{zhang2023ho} propose a KV cache eviction policy based on the summation of attention scores to only keep a small portion of the KV cache in memory. Unlike these previous methods, our method reduces the memory consumption of the KV cache by reducing the number of layers, which is orthogonal to these methods and can potentially be combined with these methods to further reduce the KV cache and improve inference efficiency.

Feedback Transformers \cite{fan2020addressing} aggregate hidden representations from all layers and use them as token memory. This is followed by KV projections to obtain the key-value pairs for all layers. Their experimental results show improved performance with this strategy, even when using only the hidden representation from the top layer. However, their sequential training process is time-costly and not practical for large models. Our method supports parallel training, which is more efficient and scalable.

\section{Conclusion}

In this paper, we propose a novel method to reduce the memory consumption and improve the throughput of LLMs by reducing the number of layers whose keys and values need to be computed and cached. We empirically show that our method achieves significant memory reduction and throughput improvement with negligible performance degradation. We also show that our method could effectively integrate with other memory-saving techniques like StreamingLLM. Future work includes designing more efficient training approaches, developing large-batch-friendly kernels, and verifying our method on larger and more complex LLMs. We hope that our work could provide a new perspective for improving inference efficiency of LLMs and inspire more research in this direction.

\section*{Limitations}

Though our method achieves impressive memory reduction and throughput improvement, we would like to point out its limitations from the following aspects:
\begin{itemize}
  \item Due to the iterative training, our model requires about $3\times$ the time to pre-train a model with the same amount of data. In other words, our method improves the inference efficiency at the cost of the training efficiency. A potential remedy is that if one has a pre-trained model, one could use it to initialize our model, which is empirically found to speed up the process of training. 
  \item Since our method requires iteratively processing the prompts, the throughput degrades when the prompts are much longer than the generation length, e.g., in document summarization. 
  Generally, our method is more suitable for tasks with a large generation length, such as translation, dialogue, question answering, CoT problem solving, etc.
\end{itemize}

\bibliography{anthology,custom}

\appendix

\section{Proof of The Training Theorem}
\label{apx:proof}

Here we formally prove Theorem~\ref{thm:train} in Section~\ref{sec:training}. Remember that we need to compute the loss for each token sequentially, but we propose to train all the tokens in parallel with $n$ iterations.

\train*

To prove the theorem, we first introduce the following lemma.

\begin{lemma}
  In the first computation graph, denote the final hidden representation of the $i$-th token as $h_i = f_{1,i}(x_1, x_2, \dots, x_i)$. In the second computation graph, denote the final hidden representation of the $i$-th token in iteration $t$ as $h^{(t)}_i = f^{(t)}_{2,i}(x_1, x_2, \dots, x_i)$. We have $f_{1,i} = f^{(t)}_{2,i}, \forall t \geq i$, no matter what the initial KVs are.
\end{lemma}

\begin{proof}
  In the first computation graph, denote the KVs of the $i$-th token as ${KV}_i$ and that under parallel training with iteration $t$ as ${KV}^{(t)}_i$.

  Since the basic networks are the same for the two computation graphs, we further denote $g_i$ as $h_i = f_{1,i}(x_1, x_2, \dots, x_i) = g_i(x_1, x_2, \dots, x_i, {KV}_1, {KV}_2, \dots, {KV}_{i-1})$, as well as $h^{(t)}_i = f^{(t)}_{2,i}(x_1, x_2, \dots, x_i) = g_i(x_1, x_2, \dots, x_i, {KV}^{(t-1)}_1, \dots, {KV}^{(t-1)}_{i-1})$. The only difference in the computation of the final hidden representations lies in the KVs.

  We prove the lemma by induction. For the base case, since we have removed the diagonal of the attention matrix, $h_1$ and $h^{(t)}_1$ does not rely on any key or value. So we have $h_1 = f_{1,1}(x_1) = g_i(x_1) = f^{(t)}_{2,1}(x_1) = h^{(t)}_1, \forall t$. That is $f_{1,1} = f^{(t)}_{2,1}, \forall t$. Since the computation graphs of the first token are the same, ${KV}_1$ and ${KV}^{(t)}_1$ are the same for all $t$.
  
  For the inductive step, we assume that $\forall i \leq T, f_{1,i} = f^{(T)}_{2,i} = g_i, {KV}_i = {KV}^{(T)}_i$. Then for iteration $T+1$:

 $\forall i \leq T+1$, the computation of the $i$-th token relies on the KVs of the first $i-1$ tokens. That is, $h^{(T+1)}_i = f^{(T+1)}_{2,i}(x_1, x_2, \dots, x_i) = g_i(x_1, x_2, \dots, x_i, {KV}^{(T)}_1, \dots, {KV}^{(T)}_{i-1})$.

  Since $i \leq T+1$, we have $i-1 \leq T$. By induction we have ${KV}^{(T)}_{i-1} = {KV}_{i-1}$. Thus, we have $g_i(x_1, x_2, \dots, x_i, {KV}^{(T)}_1, \dots, {KV}^{(T)}_{i-1}) = g_i(x_1, x_2, \dots, x_i, {KV}_1, \dots, {KV}_{i-1}) = f_{1,i}(x_1, x_2, \dots, x_i) = h_i$. The computation graphs are the same for the first $T+1$ tokens, then we have ${KV}_i = {KV}^{(T+1)}_i$.

  Thus, we have $f_{1,i} = f^{(t)}_{2,i}, \forall i \leq t$. The proof does not rely on initial KVs, so the conclusion holds for any initial KVs.
\end{proof}

Let $t = n$, we have the entire computation graphs are equivalent. Therefore, we have proved Theorem~\ref{thm:train}.

From the lemma, we could learn one more thing: The parallel training has the same computation graph as the sequential training when the iteration $t \geq n$. This indicates that the parallel training is theoretically guaranteed to converge to the same solution as the sequential training. Once it is converged, it will not diverge. Our work finds that the KVs converge much faster than the theoretical $n$ iterations, which significantly reduces the training time.

\section{Model and Training Details}
\label{apx:details}

\begin{table*}[tb]
\centering
\begin{tabular}{@{}lcccc@{}}
\toprule
Model Size         & 50M  & 1.1B & 7B    & 30B   \\ \midrule
Hidden Size        & 512  & 2048 & 4096  & 6656  \\
Intermediate Size  & 1024 & 5632 & 11008 & 17920 \\
Max Trained Length & 1024 & 2048 & --    & --    \\
\# Layers          & 8    & 22   & 32    & 60    \\
\# Attention Heads & 8    & 32   & 32    & 52    \\
\# KV Heads        & 4    & 4    & 32    & 52    \\
RMS Norm eps       & 1e-5 & 1e-5 & 1e-6  & 1e-6  \\
Vocab Size         & \multicolumn{4}{c}{32000}   \\ \bottomrule
\end{tabular}
\caption{Model configurations.}
\label{tab:model}
\end{table*}

\begin{table*}[tb]
\centering
\begin{tabular}{@{}lcccc@{}}
\toprule
Section                     & \multicolumn{2}{c}{\ref{sec:sandwich}} & \ref{sec:warmup}        & \ref{sec:convergence}          \\ \cmidrule(l){2-5} 
Model Size                  & 50M          & 1.1B     & 1.1B       & 50M          \\ \midrule
$m$                         & 7            & 7        & 7          & --           \\
$b$                         & 2            & 2        & 2          & 2            \\
$w$                         & 2            & 2        & --         & 2            \\
lr scheduler                & \multicolumn{4}{c}{cosine}                          \\
max. lr                     & 3e-4         & 3e-4     & 4e-4       & 3e-4         \\
min. lr                     & 0            & 0        & 4e-5       & 0            \\
optimizer                   & \multicolumn{4}{c}{AdamW}                           \\
$\beta_1$                   & 0.9          & 0.9      & 0.9        & 0.9          \\
$\beta_2$                   & 0.999        & 0.999    & 0.95       & 0.999        \\
batch size (tokens)         & 16K          & 256K     & 256K       & 16K          \\
warmup ratio                & 0.015        & 0.015    & 0.024      & 0.015        \\
weight decay                & 6.6e-6       & 6.6e-6   & 1e-1       & 6.6e-6       \\
gradient clipping           & 1.0          & 1.0      & 1.0        & 1.0          \\
initialize from pre-trained & yes          & yes      & no         & no           \\
epochs                      & 3            & 1        & 2B tokens  & 3            \\
Data                        & WikiText-103 & MiniPile & SlimPajama & WikiText-103 \\
GPU                         & RTX 3090x1   & A100x8   & A800x16    & RTX 3090x1   \\ \bottomrule
\end{tabular}
\caption{Training details for Section~\ref{sec:analysis}.}
\label{tab:train}
\end{table*}

We provide the model configurations and training details in Table~\ref{tab:model} and \ref{tab:train}. The 7B and 30B model configurations are consistent with those of the original Llama \cite{touvron2023llama}. The 1.1B model configuration follows that of TinyLlama \cite{zhang2024tinyllama}. We use the WikiText-103 \cite{merity2017pointer} (licensed under CC-BY-SA 3.0), MiniPile \cite{kaddour2023minipile} (licensed under MIT) and SlimPajama \cite{cerebras2023slimpajama} (various licenses depending on the data source) as our datasets. Our use of the datasets is consistent with their intended use.

The training of TinyLlama in Section~\ref{sec:experiments} takes 14:42:59, while our models take 1 day, 16:44:16 (2.77x, $w=2$) and 1 day, 15:52:38 (2.71x, $w=10$), respectively. The training took place before we optimize our codes, so the training time could be further reduced, especially for the model with $w=10$.

\section{More Analyses}

In this section, we provide more analyses beyond those in Section~\ref{sec:analysis}.

\subsection{Initialize with Pre-trained Models}

\begin{table}[tb]
\centering
\begin{tabular}{@{}lcc@{}}
\toprule
Model                   & Dev ppl. & Avg. \\ \midrule
TinyLlama               & 9.219    & 46.65 \\
Ours ($w=2$)                   & 9.746    & 45.45 \\ \midrule
\textit{TinyLlama-2.5T} & 7.822    & 53.97 \\
\textit{TinyLlama-500B} & 9.046    & 48.28 \\
Ours ($w=2$, init w/ 2.5T)     & 8.514    & 49.55 \\ \bottomrule
\end{tabular}
\caption{Dev ppl: Perplexity on a 10M subset of the validation set of SlimPajama. Avg: Average zero-shot accuracy on commonsense reasoning tasks. Models are trained on a 100B subset of SlimPajama. The models with italic fonts are from TinyLlama checkpoints.}
\label{tab:init-perplexity}
\end{table}

Since our model structure resembles that of the standard transformer, we could initialize our model with pre-trained models. For $W_K, W_V$ of the middle layers, we just ignore the parameters. Experiments show that initialization with pre-trained models could effectively speed up the training process (Table~\ref{tab:init-perplexity}). Though our model has a different computation graph from the standard transformer, it could still benefit from the pre-trained models. Our model initialized with a TinyLlama checkpoint trained on 2.5T tokens achieves a perplexity of 8.514, which is much better than the randomly initialized model. It is even better than the TinyLlama checkpoint trained on 500B tokens. Thus, if the pre-trained models are available, initializing our model with them could save a lot of training time.

\subsection{Iterations with Gradients}

In Section~\ref{sec:backpropagation} we set $b=2$ by default. What if we use larger $b$? To make sure that the inference process is equivalent (specifically, encoding the prompts), we fix $m+b=9$ and experiment with different values of $b$.

Figure~\ref{fig:iter-grad} shows the perplexity of a 50M Llama with different values of $b$. Though from $b=2$ to $b=3$ the perplexity decreases, for $b=4$ the perplexity increases. We further confirm that for $b=4$ the KVs do not converge as fast as for $b=2$.

Experiments on a 1.1B model (Table~\ref{tab:encoders-ppl}) show that the perplexity increases with $b$ and the training time also increases. From the training curve, we find that larger $b$ leads to more unstable training, thus the model is harder to converge. Therefore, we set $b=2$ by default.

\begin{figure}[tb]
  \centering
  \includegraphics[page=1,width=0.44\textwidth,trim=0 20 0 50,clip]{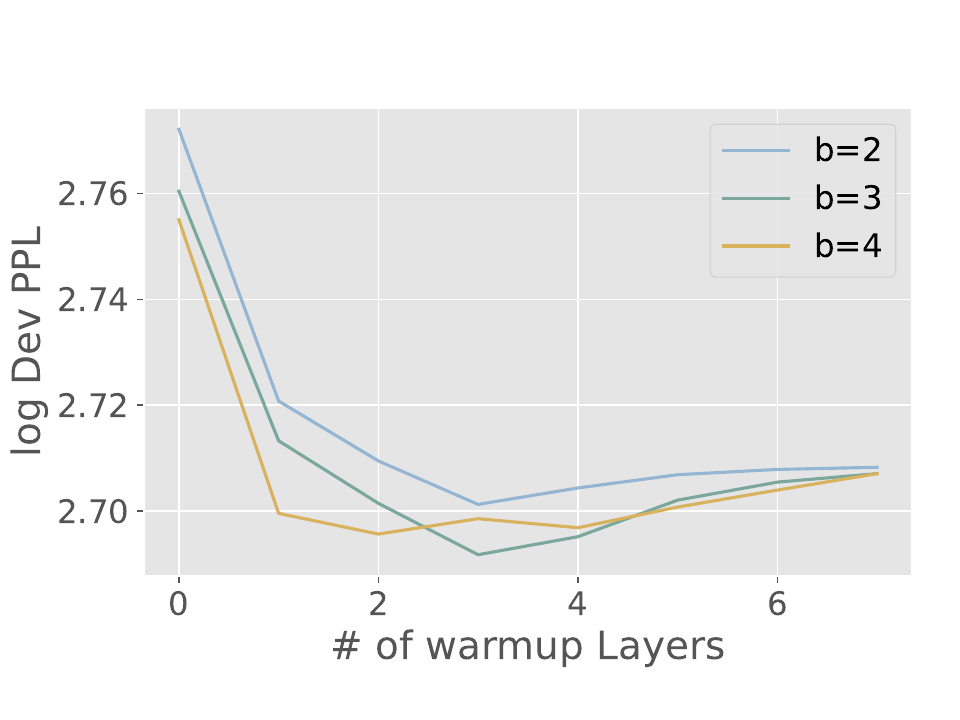}
  \caption{Effect of $b$ on a 50M model with different number of warmup layers.}
  \label{fig:iter-grad}
\end{figure}

\begin{table}[tb]
\centering
\begin{tabular}{@{}lcc@{}}
\toprule
Model                & Dev ppl. & Train Time \\ \midrule
Ours ($b=2$)         & 10.390   & 8h         \\
Ours ($b=3$)         & 10.476   & 10h        \\
Ours ($b=4$)         & 10.885   & 13h        \\ \bottomrule
\end{tabular}
\caption{Effect of $b$ on a 1.1B model.}
\label{tab:encoders-ppl}
\end{table}

\subsection{KV Loss for Less Iterations}

In Section~\ref{sec:convergence} we have shown that the KVs converge very fast for a trained model, yet we still want to make it converge even faster, saving both training and inference time costs. An intuitive idea is to add an MSE loss to the KVs before and after the last iteration to force the KVs to converge. We call this term the ``KV Loss''. Our experiments (Table~\ref{tab:kv-loss}) show that for small data with small $w$, the KV loss could lead to better performance. While for large data or large $w$, the KV loss hurts the performance. This is probably because the KVs are not converged at the beginning of training and the KV loss helps the KVs converge. However, when the KVs are already converged, the KV loss could slow down the training process. In our method, we do not use the KV loss.

\begin{table}[tb]
\centering
\begin{tabular}{@{}cccc@{}}
\toprule
Model Size            & $w$ & KV Loss & Dev ppl. \\ \midrule
\multirow{4}{*}{50M}  & 0   & no      & 15.965   \\
                      & 0   & yes     & 15.610   \\
                      & 2   & no      & 15.004   \\
                      & 2   & yes     & 15.065   \\ \midrule
\multirow{2}{*}{1.1B} & 2   & no      & 9.746    \\
                      & 2   & yes     & 10.073   \\ \bottomrule
\end{tabular}
\caption{Effect of the KV loss on a 50M and a 1.1B model. The 50M model is trained on WikiText-103 with 3 epochs and the 1.1B model is trained on a 100B subset of SlimPajama.}
\label{tab:kv-loss}
\end{table}

\subsection{Encode Prompts with Different Number of Iterations}

Though we set $m=7, b=2$ during training, it does not necessarily mean that we have to encode the prompts with 9 iterations. Is it possible to encode the prompts with less iterations? What if we encode the prompts with more iterations?

We treat the token segments as prompts and test the model trained in Section~\ref{sec:performance} with different numbers of iterations during encoding. As shown in Figure~\ref{fig:iter-inference}, the perplexity increases when the number of iterations is reduced, but still in a reasonable scale if we only reduce one or two iterations. The more warmup layers there are, the more stable the performance is. Increasing the number of iterations does not noticeably affect the performance. Thus, one could make a trade-off to set the proper number of iterations during inference to balance the time encoding prompts and the quality of generation texts.

\begin{figure}[tb]
  \centering
  \includegraphics[page=1,width=0.44\textwidth,trim=0 20 0 50,clip]{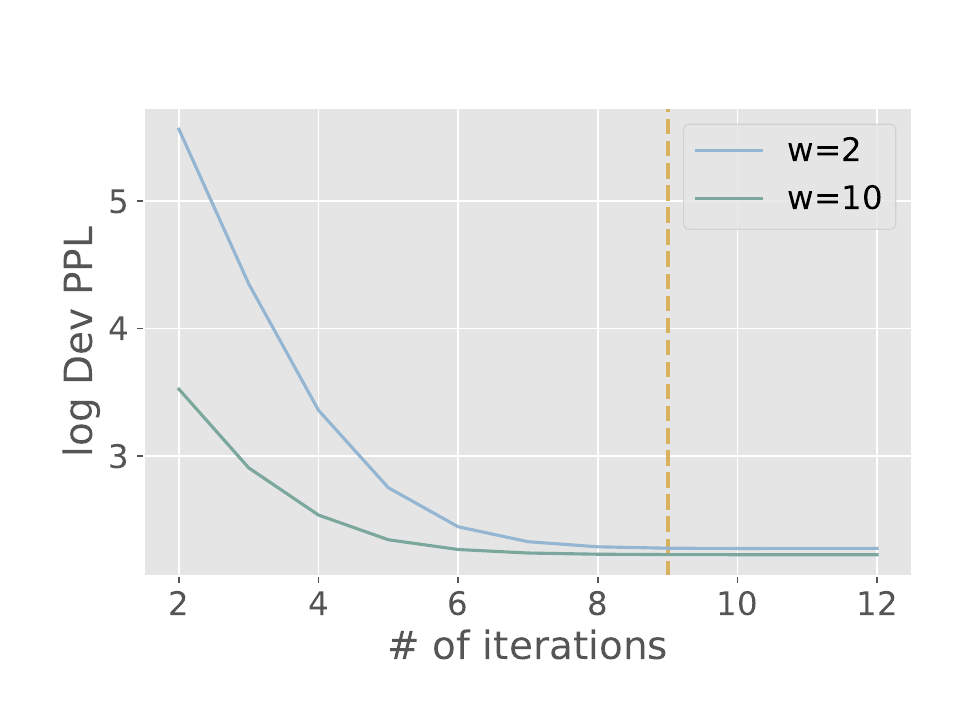}
  \caption{Performance of the models with different numbers of iterations for prompt encoding.}
  \label{fig:iter-inference}
\end{figure}

\begin{table*}[tb]
\centering
\begin{tabular}{@{}ccccc@{}}
\toprule
Model Size           & Model         & Batch Size & Latency (s) & Throughput (token/s) \\ \midrule
\multirow{3}{*}{7B}  & Llama         & 15         & 217.71      & 141.10               \\
                     & Ours ($w=2$)  & 16         & 89.15       & 367.56               \\
                     & Ours ($w=10$) & 16         & 131.43      & 249.32               \\ \midrule
\multirow{3}{*}{30B} & Llama         & 1          & 145.23      & 14.10                \\
                     & Ours ($w=2$)  & 1          & 101.12      & 20.25                \\
                     & Ours ($w=10$) & 1          & 106.99       & 19.14                \\ \bottomrule
\end{tabular}
\caption{Inference latency and throughput of the models with different model sizes. The models are tested on an A100 (80GB) GPU with prompt length 2048 and generation length 2048. The batch size is (approximately) the largest batch size for standard Llama model.}
\label{tab:latency}
\end{table*}

\subsection{Model Performance With Respect to Token Position}

The long-context performance of LLMs highly relies on KV cache \cite{xiao2024efficient,han2023lm, adnan2024keyformer}. To verify that the performance of our model does not degrade under long context, we test the perplexity of our model with different token positions on PG19 \cite{Rae2020Compressive}. Despite the fact that we only compute the KVs of a few layers, the performance of our model does not degrade with the token position (Figure~\ref{fig:token-position}) and is comparable to that of TinyLlama. Due to the limitation of computational resources, we only train models with context length 2048.

\begin{figure}[tb]
  \centering
  \includegraphics[page=1,width=0.44\textwidth,trim=0 20 20 50,clip]{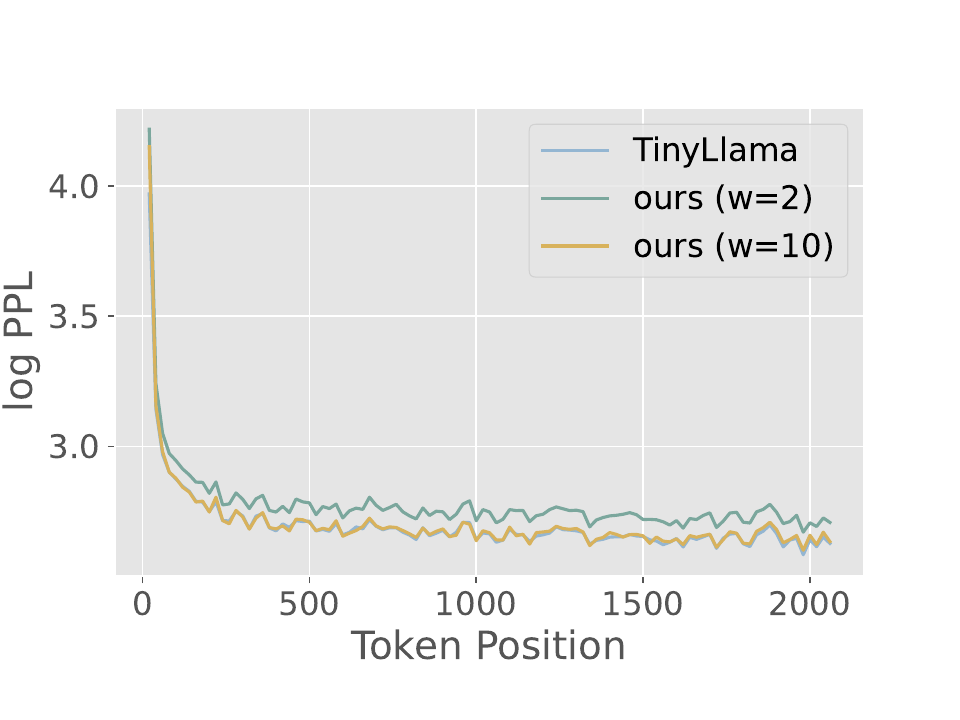}
  \caption{Perplexity of the models with different token positions on PG19. We use the concatenated test set of the PG19 dataset as the input.}
  \label{fig:token-position}
\end{figure}

\subsection{Improvement of Throughput Not Necessarily Due to Larger Batch Size}
\label{apx:throughput}

In Section~\ref{sec:throughput}, we show that our model achieves higher generation throughput than standard transformers. While one might assume that this improvement is solely due to a larger batch size, Figure~\ref{fig:throughput} demonstrates that our model outperforms the standard transformer even with the same batch size. Additionally, Table~\ref{tab:latency} reports reduced latency for our model when compared with the standard transformer with the same batch size. This suggests that our method can also benefit scenarios that require fast response. The exact reason for this phenomenon is not yet clear and we speculate that it could be attributed to factors such as the reduced calculation of KVs, the decreased memory consumption enabling faster memory transfer and access, and various implementation details.

\end{document}